\documentclass{article}

\usepackage{amsmath,amsthm,amssymb}
\usepackage{mathtools,bbm,xspace}
\usepackage[capitalize,noabbrev]{cleveref}

\usepackage[textsize=tiny]{todonotes}
\usepackage{float}

\usepackage[utf8]{inputenc} 
\usepackage[T1]{fontenc}    
\usepackage{url}            
\usepackage{amsfonts}       
\usepackage{nicefrac}       
\usepackage{microtype}      
\usepackage{xcolor}         
\usepackage{enumitem}

\makeatletter
\newcommand{\newreptheorem}[1]{\newtheorem*{rep@#1}{\rep@title}\newenvironment{rep#1}[1]{\def\rep@title{\cref*{##1} Restated}\begin{rep@#1}}{\end{rep@#1}}}
\makeatother

\theoremstyle{plain}
\newtheorem{theorem}{Theorem}[section]
\newreptheorem{theorem}

\newtheorem{lemma}[theorem]{Lemma}

\theoremstyle{definition}
\newtheorem{definition}[theorem]{Definition}

\theoremstyle{remark}

\usepackage{newtxtext}

\usepackage{ bbold }
\newcommand{\E}{\mathbb{E}}

\newcommand{\p}{\mathbb{P}}
\DeclareMathOperator{\er}{Er}
\renewcommand{\Pr}{\p}
\DeclareMathOperator{\sign}{sign}

\newcommand{\lp}{\mathopen{}\left(}
\newcommand{\rp}{\right)\mathclose{}}
\newcommand{\lk}{\left[}
\newcommand{\rk}{\right]}
\newcommand{\ind}{\mathbb{1}}
\newcommand{\deq}{\overset{\text{def}}{=}}
\renewcommand{\epsilon}{\varepsilon}
\renewcommand{\phi}{\varphi}
\let\log\relax
\DeclareMathOperator{\log}{ln}

\makeatletter

\newcommand{\alphacmd@factory}[1]{}
\newcounter{alphacmdcounter}
\newcommand{\GenerateAlphabetCmds}[2]{%
    \renewcommand{\alphacmd@factory}[1]{%
        \expandafter\providecommand\csname #1##1\endcsname{{#2{##1}}}%
    }
    \setcounter{alphacmdcounter}{0}
    \loop
        \stepcounter{alphacmdcounter}
        \edef\alphacmd@ID{\@Alph\c@alphacmdcounter}
        \expandafter\alphacmd@factory\alphacmd@ID
    \ifnum\thealphacmdcounter<26
    \repeat
}

\newcommand{\GenerateAlphabetCmdsLower}[2]{%
    \renewcommand{\alphacmd@factory}[1]{%
        \expandafter\providecommand\csname #1##1\endcsname{{#2{##1}}}%
    }
    \setcounter{alphacmdcounter}{0}
    \loop
        \stepcounter{alphacmdcounter}
        \edef\alphacmd@ID{\@alph\c@alphacmdcounter}
        \expandafter\alphacmd@factory\alphacmd@ID
    \ifnum\thealphacmdcounter<26
    \repeat
}
\makeatother

\GenerateAlphabetCmds{}{\mathbb}

\GenerateAlphabetCmds{c}{\mathcal}

\RequirePackage{times}

\RequirePackage{fancyhdr}
\RequirePackage{xcolor} 
\RequirePackage{algorithm}
\RequirePackage[noend]{algorithmic}
\usepackage{natbib}
\RequirePackage{eso-pic} 
\RequirePackage{forloop}

\usepackage{fullpage}

\date{}

\title{Improved Replicable Boosting with Majority-of-Majorities}

\author{Kasper Green Larsen \qquad Markus Engelund Mathiasen \qquad Clement Svendsen \\ \texttt{\{larsen, markusm, clementks\}@cs.au.dk}\\ \\ Aarhus University}

\begin{document}

\maketitle

\begin{abstract}
\noindent We introduce a new replicable boosting algorithm which significantly improves the sample complexity compared to previous algorithms. The algorithm works by doing two layers of majority voting, using an improved version of the replicable boosting algorithm introduced by \citet{impagliazzo2022reproducibility} in the bottom layer. 
\end{abstract}

\section{Introduction}\label{sec:introduction}
Replicability of an algorithm is a property introduced as a reaction to what is called the reproducibility crisis. Multiple Nature articles have pointed out the issue of researchers not being able to replicate findings~\citep{baker2016reproducibility, ball2023ai}.
As a supplement to implementing better research practices in order to ensure replicability, \citet{impagliazzo2022reproducibility} introduced the concept of replicability as a property of algorithms themselves. Informally, an algorithm is replicable if it, with high probability, outputs the same result when run with \textit{different} input data drawn from the same distribution.

\begin{definition}[Replicability~\citep{impagliazzo2022reproducibility}]
    Let $\cA$ be a randomized algorithm. Then, $\cA$ is said to be $\rho$-replicable if there is an $n\in \N$ such that for all distributions $\cD$ on some space $\cX$, it holds that
    \[ \Pr_{S_1,S_2,r}\left[\cA(S_1; r) = \cA(S_2; r)\right] \geq 1-\rho \]
    where $S_1,S_2 \sim \cD^n$ are independent, and $r$ denotes the internal randomness used by $\cA$.
\end{definition}
\raggedbottom
As is evident from the definition, we require the algorithm to use the same internal randomness $r$ in both runs. This turns out to be crucial - if we remove this requirement, we cannot solve simple tasks such as estimating the mean of a distribution replicably \citep{dixon2024list}. Researchers who use replicable algorithms may then publish the random seed used in their run of the algorithm which lets other researchers use the same seed to replicate the results with high probability, assuming that the data they use comes from the same underlying distribution.

In this work, we consider replicability in the weak-to-strong learning setting. Specifically we improve the best known sample complexity of $\rho$-replicable boosting algorithms. 
Let $\cX$ be an input domain, and let $f: \cX \rightarrow \{-1,1\}$ be the function we are trying to predict.
An algorithm $\cW$ is said to be a $\gamma$-weak learner for $\gamma \in (0,1/2)$ if there exists an $ m\in \mathbb{N}$ such that for any distribution $\cD$ on $\cX$ and any sequence of $m$ labelled samples $S = \{(x_i, f(x_i))\}_{i=1}^{m}$ drawn i.i.d.\ from $\cD$, it holds that $h\coloneqq \cW(S) : \cX \to \{-1,1\}$ satisfies $\Pr_{x\sim \cD}[h(x)\neq f(x)] \leq 1/2 - \gamma.$ We call $\gamma$ the \textit{advantage} of $\cW$ and $m$ the \textit{sample complexity} of $\cW$.
A strong learner on the other hand, is a learning algorithm such that for any distribution $\cD$ on $\cX$, failure probability $\delta>0$ and error $\epsilon>0$, there is an $m=m(\epsilon,\delta) \in \mathbb{N}$ such that when applied to an i.i.d.\ sample $S \sim \cD^m$, the algorithm outputs a classifier which has error at most $\epsilon$ over $\cD$ with probability at least $1-\delta$. We denote the error of a hypothesis $h: \cX \to \{-1,1\}$ with respect to a distribution $\cD$ by $\er_{\cD}(h) = \p_{x\sim\cD}[h(x) \ne f(x)]$.

Boosting algorithms were originally introduced to answer the following theoretical question posed by \citet{kearns1988learning,kearns1994cryptographic}: Is it possible to combine hypotheses produced by a weak learning algorithm into a strong learner? As shown by \citet{boost}, this turned out to be the case, and one of the most famous algorithms that solves this problem is the \textsc{AdaBoost} algorithm \citep{Adaboost}. In short, boosting works by running a number of iterations. In each iteration $t$, we update a distribution $\cD_t$ on the samples and run the weak learner with this new distribution. After a sufficient number of iterations, we take a weighed majority vote among all the produced weak hypotheses.

\subsection{Our contribution}
Our main contribution is a replicable boosting algorithm called \textsc{rMetaBoost} which is inspired by an existing replicable boosting algorithm, \textsc{rBoost} \citep{impagliazzo2022reproducibility} and the \textsc{SmoothBoost} algorithm \citep{Smoothboost}. First, fix a distribution $\cD$ on $\cX$ and let $\cW$ denote a replicable weak learner and for $\rho\in (0,1)$ let $m_{\cW(\rho)}$ be the sample complexity of $\cW$ when run with replicability parameter $\rho$. Our main result is the following:
\begin{theorem}[\textsc{rMetaBoost}]\label{thm:rMetaBoost}
For any $\rho, \epsilon \in (0,1)$ and $\widetilde{\Theta}(\rho\gamma^2)$-replicable weak learner $\cW$ with advantage $\gamma$, \textsc{rMetaBoost} is $\rho$-replicable, makes $O(\frac{\log(1/\epsilon)}{\gamma^2})$ calls to $\cW$, and with probability at least $1-\rho$ outputs a hypothesis $H$ with $\er_\cD(H) \leq \epsilon$. Furthermore, its sample complexity is
\[\widetilde{O}\lp \frac{m_{\cW(\widetilde{\Theta}(\rho \gamma^2))}}{\epsilon\gamma^2} + \frac{1}{\rho^2\epsilon\gamma^3}\rp.\]
\end{theorem}
Our algorithm significantly improves on the sample complexity of \citet{impagliazzo2022reproducibility} which is \medbreak \noindent $\widetilde{O}\lp\frac{m_{\cW(\Theta(\rho\epsilon\gamma^2))}}{\epsilon^2\gamma^2}+ \frac{1}{\rho^2\epsilon^5\gamma^6}\rp$. 
Note that this sample complexity is not what is stated in their paper, but is in fact the correct sample complexity of their algorithm.\footnote{We have personally contacted the authors to make them aware, and they have acknowledged this error.}
We improve the first term by a factor $1/\epsilon$ and also remove a factor $\epsilon$ in the replicability parameter to the weak learner. Since the sample complexity of most replicable algorithms has a quadratic dependence on their replicability parameter, this will amount to an extra $1/\epsilon^2$ improvement in this term. In the second term we shave off a factor $1/(\epsilon^4\gamma^3)$. All improvements are up to logarithmic factors.

As a secondary contribution, we introduce an algorithm \textsc{rThreshold} for performing a replicable \textit{threshold check}. This algorithm replicably checks if the expected value of a function $\phi$ is above a certain threshold $z$, and is used as a subroutine in \textsc{rMetaBoost}. We state the guarantees of the algorithm below. 

\newcommand{\rThreshold}{
    Let $z,\rho \in (0, 1)$, $\delta \in (0,\rho/8]$, let $\phi: \cX \rightarrow [0,1]$ and let $S=(x_1, \ldots, x_m)$ be samples drawn i.i.d.\ from distribution $\cD$. Then there exists a constant $c$ such that if
    \(m \geq c \frac{\log(1/\delta)}{\rho^2z}\),
    \textsc{rThreshold}$(S, z,\phi)$ is $\rho$-replicable and returns a bit $b$ such that with probability at least $1 - \delta$:
    \begin{itemize}[noitemsep,topsep=0pt]
        \item If $\E_{x\sim \cD}[\phi(x)] \leq z/2$, then $b = 0$.
        \item If $\E_{x\sim \cD}[\phi(x)] \geq 2z$, then $b = 1$.
    \end{itemize}
}
\begin{lemma}[\textsc{rThreshold}]\label{lem:rThreshold}
\rThreshold
\end{lemma}

We believe this algorithm is also of independent interest and can be applied in many scenarios as an alternative to statistical queries which were previously used for such applications. This is because our algorithm achieves a dependence of $1/z$ in the sample complexity, while using statistical queries for the same purpose comes with a factor $1/z^2$ in the sample complexity (Thm. 2.3, ~\citet{impagliazzo2022reproducibility}). While our approach to threshold checks is not neccesarily novel, it seems to have been overlooked in the context of replicable algorithms.

\subsection{Related work}
In recent years, replicable algorithms have been developed in a variety of settings. This includes e.g.\ learning half spaces, clustering, reinforcement learning and online learning \citep{Half_spaces, clustering, RL, Online}.

There are also important connections to the field of differential privacy. Intuitively, a replicable algorithm does not depend heavily on the specific sample given to the algorithm. This is similar to the requirement in differential privacy where we demand that when the algorithm is run on two samples differing in only a single point, then the two distributions on the outputs are close in the sense of max divergence.
\citet[Thm.~3.1]{DP} show that there is a reduction "without substantial blowup in runtime or sample complexity" from differential privacy to replicability. On the other hand, they also show that no computationally efficient transformation of differentially private algorithms to replicable ones can exist under standard cryptographic assumptions. However, if one does not care about computational efficiency, they do give a reduction from differential privacy to replicability with only a quadratic blowup in sample complexity. This means it would be possible to take an existing differentially private boosting algorithm and make it replicable. One example of a differentially private boosting algorithm is \textsc{BoostingForPeople} \citep{dwork2010boosting}. However, using the reduction on this algorithm would incur a $1/\gamma^8$ and $1/\epsilon^2$ dependence in the sample complexity.

Moving away from differential privacy, another candidate algorithm to be made replicable is the \textsc{SmoothBoost} algorithm \citep{Smoothboost}. This algorithm differs from e.g.\ the well-known \textsc{AdaBoost} \citep{Adaboost} in that it maintains a \textit{smoothness} across the distributions $\cD_t$ over the data in every iteration $t$.
Formally, this means that the distribution $\cD_t$ satisfies $\max_{x} \cD_t(x)\le 1/(\epsilon m)$ for some $\epsilon>0$ where $m$ is the number of samples. This smoothness property ensures that no single example has too much influence on the distributions which is why smoothness is a desirable property when designing replicable boosting algorithms. In fact, the boosting algorithm by \citet{impagliazzo2022reproducibility} can be seen as a translation of \textsc{SmoothBoost} into the replicable setting.

The downside of using \textsc{SmoothBoost} is that it requires $O(\frac{1}{\epsilon\gamma^2})$ invocations of the weak learner $\cW$. We call this the \textit{round complexity} of the algorithm. This should be compared to \textsc{AdaBoost} which has round complexity $O(\frac{\log(1/\epsilon)}{\gamma^2})$. In the replicable setting, we draw new samples for each invocation of $\cW$, so the round complexity directly affects the number of samples used. This motivates looking at smooth boosting algorithms with fewer invocations of $\cW$ such as the one presented by \citet{Hardcore}. This algorithm uses Bregman projections to maintain the smoothness property, and it matches the round complexity of \textsc{AdaBoost}. However, converting the algorithm to the replicable setting would require us to make replicable approximations of these Bregman projections which turns out to use more samples than we obtain in \cref{thm:rMetaBoost}.

\subsection{High-Level Ideas}
We will now explain the very high-level idea behind our new boosting algorithm \textsc{rMetaBoost}.
The first step towards constructing this improved replicable boosting algorithm is to make slight modifications to the algorithm \textsc{rBoost} of \citet{impagliazzo2022reproducibility} to improve its sample complexity. We will refer to this modified version as \textsc{rBoost}$^*$ which can be found in \cref{alg:rBoost}.
Remark that the functions $g_t, \mu_t$ are functions over the entire domain $\cX$ and not just the samples that we see. This means that we cannot afford to update these functions explicitly for every point, so instead we update the description of the functions. To distinguish this from normal assignments in the pseudocode, we use the $\deq$ operator for assignments to these functions and the $\gets$ operator for normal assignments.

In this algorithm $\mu_t: \cX \rightarrow [0,1]$ is a function which determines the reweighing of the data distribution $\cD$ in iteration $t$. The reweighed distribution is then $\cD_{\mu_t}(x) = \mu_t(x)\cD(x)/d(\mu_t)$ where $d(\mu_t) = \E_{x\sim \cD}[\mu_t(x)]$ is the normalization factor which we call the \textit{density} of $\mu$. The subroutine \textsc{RejectionSampler} then lets us sample from the distribution $\cD_{\mu_t}$ when given access to $\mu_t$ and samples from $\cD$ (see \cref{lem:rejection} for formal guarantee). We also note without proof that large density of $\mu_t$ actually implies smoothness of the reweighed distribution $\cD_{\mu_t}$ with respect to the original distribution $\cD$. More precisely, if $d(\mu_t) \geq \epsilon$, for some $\epsilon>0$, then $\cD_{\mu_t}(x) \leq \cD(x)/\epsilon$ for all $x\in\cX$. These samples from $\cD_{\mu_t}$ are then given to the replicable weak learner. We will not go into further detail with how or why the original \textsc{rBoost} works but instead refer to \citet{impagliazzo2022reproducibility,Smoothboost}.

In total, we have made two modifications in \textsc{rBoost}$^*$.
The first modification is that we have changed the termination condition in line \ref{alg:rBoost:threshold} to use our \textsc{rThreshold} algorithm instead of the statistical query algorithm they used. This accomplishes exactly the same thing, but uses a factor $1/\epsilon$ fewer samples for each call.
The second modification is the introduction of the if-statement in line \ref{alg:rBoost:divides}. It turns out that this check only makes the algorithm run for a constant factor more iterations. However, this allows us to shave off a factor $1/\gamma$ in the number of calls to \textsc{rThreshold}. Since the replicability parameter of \textsc{rThreshold} needs to be $\rho$ divided by the number of calls to \textsc{rThreshold}, this is a great improvement. This is because the sample complexity of \textsc{rThreshold} is inversely proportional to the square of its replicability parameter, so it will need a factor $1/\gamma^2$ fewer samples for each invocation of \textsc{rThreshold}. Since it is now only called every $1/\gamma$ iteration, we shave off a factor $1/\gamma^3$ in total by introducing this check. In \cref{sec:subroutines}, we will explain in more detail why these modifications preserve correctness, but for now we will just state the guarantees of \textsc{rBoost}$^*$.

\begin{algorithm}[H]
    \caption{\textsc{rBoost}$^*_{\rho,\epsilon}(S, \cW)$}\label{alg:rBoost}
    \begin{algorithmic}[1]
        \REQUIRE Samples $S$ i.i.d.\ from $\cD$, replicable $\gamma$-weak learner $\cW$, replicability $\rho$, error $\epsilon$.
        \ENSURE Hypothesis $H : \cX \to \{-1,1\}$.
        \STATE $g_0(x) \deq 0$
        \STATE $\mu_1(x) \deq 1$
        \STATE $t \gets 0$
        \WHILE{\TRUE}
            \STATE $t \gets t+1$
            \STATE $\cD_{\mu_t}(x) \deq \mu_t(x)\cD(x)/d(\mu_t)$
            \STATE $S_1 \gets \widetilde{O}(m_{\cW(\Theta(\rho\epsilon\gamma^2))}/\epsilon)$ fresh samples from $S$
            \STATE $S_{\cW} \gets \textsc{RejectionSampler}(S_1, m_{\cW(\Theta(\rho\epsilon\gamma^2))}, \mu_t)$
            \STATE $h_t \gets$ Run $\cW(S_\cW)$ with replicability $\Theta(\rho\epsilon\gamma^2)$
            \STATE $g_t(x) \deq g_{t-1}(x) + h_t(x)f(x) - \gamma/(2+\gamma)$
            \STATE $\mu_{t+1}(x) \deq \begin{cases}
                1, & \text{if } g_t(x) \leq 0\\
                (1-\gamma)^{g_t(x)/2}, & \text{if } g_t(x) > 0
            \end{cases}$
            \IF{$\lfloor\frac{1}{\gamma}\rfloor$ divides $t$}\label{alg:rBoost:divides}
                \STATE $S_2 \gets \widetilde{O}\lp \frac{1}{\rho^2\epsilon^3\gamma^2}\rp$ fresh samples from $S$
                \IF{$\textsc{rThreshold}(S_2, \epsilon/2, \mu_t) = 0$}\label{alg:rBoost:threshold}
                    \STATE Exit while loop %\tcp*{failure rate $\rho/(12 \gamma T)$, query $\phi=\mu_t$}
                \ENDIF
            \ENDIF
        \ENDWHILE
        \STATE \textbf{Return:} $H \gets \sign(\sum_t h_t)$
    \end{algorithmic}
\end{algorithm}

\newcommand{\rBoost}{
    For any $\rho,\epsilon\in(0,1)$ and $\Theta(\rho\epsilon\gamma^2)$-replicable weak learner $\cW$ with advantage $\gamma$, \textsc{rBoost}$^*$ is $\rho$-replicable, makes $O(\frac{1}{\epsilon\gamma^2})$ calls to $\cW$, and with probability at least $1-\rho$ outputs a hypothesis $H$ with $\er_\cD(H) \leq \epsilon$. Furthermore, its sample complexity is 
    \begin{align*}
        m_{\textsc{rBoost}^*}(\rho,\epsilon) & = O\lp \frac{\log(\frac{1}{\rho\epsilon\gamma^2})m_{\cW(\Theta(\rho\epsilon\gamma^2))}}{\epsilon^2\gamma^2} + \frac{\log(\frac{1}{\rho\epsilon\gamma})}{\rho^2\epsilon^4\gamma^3}\rp\\
        &=\widetilde{O}\lp\frac{m_{\cW(\Theta(\rho\epsilon\gamma^2))}}{\epsilon^2\gamma^2} + \frac{1}{\rho^2\epsilon^4\gamma^3}\rp.
    \end{align*}
}
\begin{theorem}[\textsc{rBoost}$^*$]\label{thm:rBoost}
    \rBoost
\end{theorem}

Remember that the original version of \textsc{rBoost} had sample complexity $\widetilde{O}(\frac{m_{\cW(\Theta(\rho\epsilon\gamma^2))}}{\epsilon^2\gamma^2} + \frac{1}{\rho^2\epsilon^5\gamma^6})$. The dependence on $\gamma$ has therefore improved greatly. However, we are still not happy with the dependence on $\epsilon$. The main idea of our algorithm is therefore to use \textsc{rBoost}$^*$ as a subroutine and only call it with constant error parameter $\epsilon_0 = 1/16$. Our algorithm can be seen as a meta boosting algorithm where in each iteration, we call \textsc{rBoost}$^*$ to get a hypothesis with constant advantage. We then perform exponential weight updates similar to \textsc{AdaBoost} in order to make our algorithm only run for $T = O(\log(1/\epsilon))$ iterations. Remark that this entirely removes the problem of \textsc{rBoost}$^*$ having a bad dependence on $\epsilon$, since we only invoke it with a constant error parameter. This is the main insight needed to understand how our algorithm works.

\section{Our Replicable Boosting Algorithm}\label{sec:replicable_boosting}
In this section, we will present our new $\rho$-replicable boosting algorithm which can be found in \cref{alg:NewRBoost}.

The algorithm runs for $T$ iterations while maintaining functions $N_t, M_t, \mu_t$. In each iteration the algorithm performs rejection sampling to get samples $S_2$ drawn from distribution $\cD_{\mu_t}$. It then gets a hypothesis $h_t$ from \textsc{rBoost}$^*$ which has constant error of at most $\epsilon_0$ with respect to $\cD_{\mu_t}$. One can interpret $N_t(x)$ as a lower bound for counting how many of the first $t-1$ hypotheses that misclassify element $x$. However, in order to ensure high density of the updated reweighing function $\mu_t$, we check if the points sharing the largest count have a total probability mass of at least $\epsilon/16$ by using \textsc{rThreshold}. If not, we subtract $1$ from the largest count which suffices to ensure high density of $\mu_t$ (see \cref{lem: smoothness of measures}). The capped values are stored in $M_t$, and will be used in subsequent iterations. The value $c_t$ can be interpreted as a bound for the largest allowed count in iteration $t$, that is $c_t \geq M_t(x)$ for all $x \in \cX$.

\begin{algorithm}
    \caption{\textsc{rMetaBoost}$_{\rho, \epsilon}$($S$, $\cW$)}\label{alg:NewRBoost}
    \def\NoNumber#1{{\def\alglinenumber##1{}\State #1}\addtocounter{ALG@line}{-1}}
    \begin{algorithmic}[1]
        \REQUIRE Samples $S$ i.i.d.\ from $\cD$, replicable $\gamma$-weak learner $\cW$, replicability $\rho$, error $\epsilon$.
        \ENSURE Hypothesis $H: \cX \rightarrow \{-1, 1\}$.
        \STATE $N_1(x) \deq 0$
        \STATE $M_1(x) \deq 0$
        \STATE $\mu_1(x) \deq 1$
        \STATE $c_1 \gets 0$
        \FOR[$T=O(1/\epsilon)$]{$t=1$ to $T$}
            \STATE $\cD_{\mu_t}(x) \deq \mu_t(x)\cD(x)/d(\mu_t)$
            \STATE $S_1 \gets \widetilde{O}(m_{\textsc{rBoost}^*}(\rho_0,\epsilon_0)/\epsilon)$ fresh samples from $S$\label{alg:NewRBoost:rejection_sample}
            \item[]\COMMENT{\texttt{$\rho_0 = \rho/(6T)$, $\epsilon_0=1/16$}}
            \STATE $S_2 \gets \textsc{RejectionSampler}(S_1, m_{\textsc{rBoost}^*}(\rho_0, \epsilon_0), \mu_t)$
            \addtocounter{ALC@line}{1}
            \item $h_t \gets \textsc{rBoost}^*_{\rho_0,\epsilon_0}(S_2,\cW)$
            \STATE $N_{t+1}(x) \deq M_t(x)+\ind\{h_t(x) \neq f(x)\}$
            \STATE $S_3 \gets \widetilde{O}(\frac{1}{\rho^2\epsilon})$ fresh samples from $S$\label{alg:NewRBoost:threshold_sample}
            \STATE\label{alg:NewRBoost:threshold} $b_{t+1} \gets \textsc{rThreshold}(S_3, \epsilon/16, \phi)$
            \item[] \COMMENT{\texttt{$\phi(x) = \ind\{N_{t+1}(x)=c_t+1\}$}}
            \STATE $c_{t+1} \gets c_t + b_{t+1}$
            \STATE $M_{t+1}(x) \deq \min \lp N_{t+1}(x), c_{t+1}\rp$
            \STATE $\mu_{t+1}(x) \deq \exp(M_{t+1}(x) - c_{t+1})$
        \ENDFOR
        \STATE \textbf{Return} $H=\sign\lp\sum_{t=1}^T h_t\rp$
    \end{algorithmic}
\end{algorithm}

Now, before going into the analysis of the algorithm, we will present the guarantees of the \textsc{RejectionSampler} which we use to draw samples from $\cD_\mu$. The pseudocode and proofs of the below guarantees are described by \citet{impagliazzo2022reproducibility}, so we will not repeat those here.

\begin{lemma}[Rejection Sampling \citep{impagliazzo2022reproducibility}]\label{lem:rejection}
    For any $\epsilon \in (0, 1]$,
    if $\mu$ has density $d(\mu) \geq \epsilon$ and $S\sim \cD^m$ where $m\ge 8\log(1/\delta) m_{target}/\epsilon$, then \textsc{RejectionSampler}$(S,m_{target},\mu)$ outputs a sample $S_{out} \sim \cD_\mu^{m_{target}}$ with probability at least $1 - \delta$.
\end{lemma}
\begin{lemma}[Composing Replicable Algorithms with Rejection Sampling \citep{impagliazzo2022reproducibility}]\label{lem:rejection_composition}
    Let $\cA(S)$ be a $\rho$-replicable algorithm with sample complexity $m$. Let $\mu: \cX \rightarrow [0, 1]$. Then let $\cB$ be the algorithm that runs $\cA$ with samples drawn from $\cD_\mu$ using rejection sampling.
    Let $q$ be the failure probability of \textsc{RejectionSampler}. Then $\cB$ is $(2q + 2\rho)$-replicable.
\end{lemma}

\subsection{Analysis of \textsc{rMetaBoost}}
We are now ready to analyze \textsc{rMetaBoost}. To make it easier to follow the analysis, we will split it into four parts.
\begin{enumerate}[noitemsep, topsep=0pt]
    \item Correctness,
    \item Replicability,
    \item Sample complexity,
    \item Failure probability.
\end{enumerate}
We will start with correctness. However, before going into the formal details, we will give an explanation of the high level ideas in the proof. First, observe that if we did not cap the weights $N_t$, the multiplicative weight updates would be very similar to the updates made in \textsc{AdaBoost}. Recall that for any $t\in [T]$, $M_t(x)$ is exactly the number of misclassifications of $x$ minus the amount of times we have capped the weight so far. Hence, if we did not cap the weights by $c_t$ each iteration, $x$ would be misclassified by the final hypothesis $H$ only if $M_{T+1}(x) \geq T/2$.
We will now take the capping into account. We first show using an argument similar to the standard analysis of \textsc{AdaBoost} that the probability of drawing an $x$ from $\cD$ for which $M_{T+1}(x) \geq T/4$ is at most $\epsilon/2$.
What remains is to argue that the probability of drawing an $x$ which is misclassified but simultaneously satisfies $M_{T+1}(x) < T/4$ is small. The only way this can happen is if there were at least $T/4$ iterations in which we capped down the value of $N_{t+1}(x)$ when calculating $M_{t+1}(x)$, since in such iterations we would not increment $M_{t+1}(x)$ even though $h_t$ misclassified $x$.
Observe that due to the threshold check, the total probability mass (w.r.t.\ $\cD$) of points whose value of $N_{t+1}$ were capped in a single iteration cannot exceed $\epsilon/8$. Therefore, after $T$ iterations, the total probability mass of points, whose value of $N_{t+1}$ were capped $T/4$ times is at most $T(\epsilon/8)/(T/4) = \epsilon/2$. So in total, the probability mass of all the misclassified points is at most $\epsilon$. We now prove this formally.

\begin{lemma}[Correctness]
    Put $T\ge 8\log(2/\epsilon)$  and $\epsilon_0 = 1/16$. Assuming that all subroutines of the algorithm succeed in every iteration, we achieve an error of at most $\epsilon$ over the distribution $\cD$, i.e.\ $\er_{\cD}(H) \leq \epsilon$.    
\end{lemma}
\vspace{-1em}
\begin{proof}
    By definition of $M_{T+1}, N_{T+1}$ and $\mu_T$
    \begin{align}
        & \E[\exp(M_{T+1}(X))] \le \E[\exp(N_{T+1}(X))] \notag \\
        & = \E[\exp(M_T(X))\exp(\ind\{h_T(X)\ne f(X)\})] \notag \\ 
        & = e^{c_T}\E[\mu_T(X)\exp(\ind\{h_T(X)\ne f(X)\})] \notag \\
        & = e^{c_T}\big(\E[\mu_T(X)\ind\{h_T(X)=f(X)\}] \notag \\
         & \quad +e\E[\mu_T(X)\ind\{h_T(X)\ne f(X)\}]\big). \label{eq: correctness}
    \end{align}
    Now, since $\cD_{\mu_T} = \frac{\mu_T \cdot \cD }{d(\mu_T)}$, we can rewrite the above to an expectation involving $Y \sim \cD_{\mu_T}$ such that \eqref{eq: correctness} is equal to
    \begin{align}
        & e^{c_T}d(\mu_T)\big( \E[\ind\{h_T(Y)=f(Y)\}] \notag \\
        & \quad + e\E[\ind\{h_T(Y)\ne f(Y)\}] \big) \notag \\
        & = e^{c_T}d(\mu_T)\lp \er_{\cD_{\mu_T}}(h_T)(e-1) + 1 \rp \notag \\
        & \le e^{c_T}d(\mu_T) \exp\lp(e-1)\er_{\cD_{\mu_T}}(h_T)\rp \notag \\
        & \le e^{c_T} d(\mu_T) \exp(2\epsilon_0), \label{eq: recurse}
    \end{align}
    where the final inequality follows since $h_t$ has error at most $\epsilon_0$ under $\cD_{\mu_T}$ by \cref{thm:rBoost}.
    Now, note that 
    \begin{align*}
        e^{c_T} d(\mu_T) &= e^{c_{T}}\E[\mu_T(X)] = e^{c_T}\E[\exp(M_T(X)-c_T)]\\
        &= \E[\exp(M_T(X))].
    \end{align*}
    Plugging this into \eqref{eq: recurse}, we recursively get
    \begin{align*}
        \E[\exp(M_{T+1}(X))] & \le \E[\exp(M_{T}(X))]\exp(2\epsilon_0) \\
        & \le \cdots \le \exp(2T\epsilon_0)
    \end{align*}
    Now, define the sets $A = \{x : M_{T+1}(x) \ge T/4\}$ and $B = A^c \cap \{x : H(x) \ne f(x)\}$ and note that 
    \[
   \er_\cD(H) \le \p(X\in A)+\p(X\in B).
    \]
    For bounding $\p(X\in A)$, we have
    \begin{align*}
        \E\left[\exp\lp M_{T+1}(X)\rp\right] & \ge \E\left[\exp\lp M_{T+1}(X)\rp \ind_A(X) \right] \\
        & \ge \exp(T/4)\p(X\in A),
    \end{align*}

    and hence 
    \begin{align*}
        \p(X\in A) & \le \exp(-T/4)\E\left[\exp\lp M_{T+1}(X)\rp\right] \\
        & \le \exp(T\lp2\epsilon_0-1/4\rp) \\
        & = \exp(-T/8) \leq \epsilon/2.
    \end{align*}
    Bounding $\p(X\in B)$:

   \noindent First, observe that $0\le M_{t+1}(x)-M_{t}(x) \le 1$ for all $t\in [T], x \in \mathcal{X}$. Now, let $x\in B$. Since $H$ is a majority classifier, we have
    \begin{align*}
        T/2 & \le \sum_{t=1}^{T} \ind\{h_t(x) \ne f(x)\} \\
        & = \sum_{t=1}^{T} \ind\!\left\{N_{t+1}(x) > c_{t+1}\right\} + M_{T+1}(x). 
    \end{align*}
    Taking the expectation over the event $\{X\in B\}$ on both ends of the above then yields
    \begin{align*}
        T&\p(X\in B)/2 = T\E[\ind\{X\in B\}]/2 \\
         \le &\sum_{t=1}^T\E[\ind\{N_{t+1}(X) > c_{t+1}\}\ind\{X\in B\}] \\
         &+ \E[M_{T+1}(X)\ind\{X\in B\}] \\
         \le &\sum_{t=1}^T\p[N_{t+1}(X) > c_{t+1}] +\E[M_{T+1}(X)\ind\{X\in B\}] \\
        \le &\sum_{t=1}^T\p[N_{t+1}(X) > c_{t+1}] + T\p(X\in B)/4. 
    \end{align*}
    Due to the threshold check in line \ref{alg:NewRBoost:threshold} and \cref{lem:rThreshold}, we know that if $b_t=0$, then we must have $\p[N_{t+1}(X)=c_{t}+1] \leq  \epsilon/8$.
    Furthermore, in this case $c_{t+1} = c_{t}$. Hence,
    \begin{align*}
        \p[N_{t+1}(X) > c_{t+1}] & = \p[N_{t+1}(X) > c_t] \\
        & = \p[N_{t+1}(X) = c_t+1] \leq \epsilon/8.
    \end{align*}
    If instead $b_t=1$, then 
    \[
    N_{t+1}(x) \le M_{t}(x)+1\le c_{t} + 1 =c_{t+1},
    \]
    which implies 
    \[
    \p[N_{t+1}(X) > c_{t+1}] = 0.
    \]
    Hence, we get the bound
    \begin{align*}
        &T\p(X\in B)/2 \\
         & \le \sum_{t=1}^T\p[N_{t+1}(X) > c_{t+1}] + T\p(X\in B)/4  \\
        & \le T\epsilon/8 +  T\p(X\in B)/4.
    \end{align*}
    Rearranging gives $\p(X\in B) \leq \epsilon/2$, meaning we in total have
    \[
        \er_\cD(H) \leq \p(X\in A) + \p(X\in B) = \epsilon/2 + \epsilon/2 = \epsilon. \qedhere
    \]
\end{proof}

For the remaining parts, we need the guarantee of Lemma \ref{lem:rejection} that rejection sampling fails with low probability when $\mu_t$ has large density. Hence, we first show that the density is indeed large.

\begin{lemma}[High density of $\mu_t$] \label{lem: smoothness of measures}
    Assume that \textsc{rThreshold} succeeds in every iteration in \textsc{rMetaBoost}. Then for any $t\in [T]$, $\mu_t$ has density $d(\mu_t)\geq\epsilon/32$.
\end{lemma}
\begin{proof}
    Let $X\sim \cD$. Then using the law of total expectation and the definition of $\mu_t$ we have
    \begin{align*}
        d(\mu_t) & = \E[\mu_t(X)]\\
        &\ge \E[\mu_t(X) | M_t(X)\geq c_t]\p[M_t(x) \ge c_t] \\
        & = \E[\exp(M_{t}(X)-c_{t}) | M_t(X)\geq c_t]\p[M_t(x) \ge c_t] \\
        & \ge \p[M_t(X)\ge c_t].
    \end{align*}
    We now show by induction in $t$ that $\p[M_t(X)\ge c_t]>\epsilon/32$. For $t=1$, we have $M_1(X)=c_1=0$, and hence $\p[M_t(X)\ge c_t] = 1$. Now, assume the claim holds for $t$. We will then show that it holds for $t+1$ by case analysis on $b_{t+1}$. If $b_{t+1}=0$, we have $c_{t+1}=c_t$ and
    \begin{align*}
        \p[M_{t+1}(X)\ge c_{t+1}] & = \p[M_{t+1}(X)\ge c_{t}] \\
        &  \ge \p[M_{t}(X)\ge c_{t}]  > \epsilon/32
    \end{align*}
    using the induction hypothesis and the fact that $M_{t+1}(X)\ge M_t(X).$ Now assume $b_{t+1}=1$. Then we know by \cref{lem:rThreshold} that $\p\lk N_{t+1}(X) = c_t + 1\rk > \epsilon/32$. Since $b_{t+1}=1$, then $c_{t+1}=c_t+1$ which then implies that
    \begin{align*}
        & \p[M_{t+1}(X)\ge c_{t+1}] = \p[M_{t+1}(X)\ge c_t+1] \\
        & = \p[\min(N_{t+1}(X),c_{t+1}) \ge c_t+1] \\
        & = \p[N_{t+1}(X) \ge c_t+1] > \epsilon/32. \qedhere
    \end{align*}
\end{proof}

\begin{lemma}[Replicability]\label{lem:replicability}
    \textsc{rMetaBoost} is $\rho$-replicable.
    \begin{proof}
        Let $S_1,S_2$ be two independent samples with distribution $\cD^m$ for some $m = \widetilde{O}\lp\frac{m_{\cW(\widetilde{\Theta}(\rho\gamma^2))}}{\epsilon\gamma^2} + \frac{1}{\rho^2\epsilon\gamma^3}\rp$. Assume that in iterations $1,\dots,t-1$, the algorithm has produced the same objects, i.e.\ that the reweighing functions and hypotheses associated with $S_1$ and $S_2$ are the same. Then, for iteration $t$ to be replicable, we need the following: 
        \begin{enumerate}[noitemsep,topsep=0pt]
            \item \textsc{rBoost}$^*$ outputs the same hypothesis for both samples.
            \item \textsc{rThreshold} outputs the same bit for both samples.
        \end{enumerate}
        When these conditions hold, the rest of the quantities appearing in the algorithm will be the same for both samples and hence ensure replicability. We call \textsc{rBoost}$^*$ with replicability parameter $\rho_0 = \rho/(6 T)$ and call \textsc{RejectionSampler} with at least $8 \log(6T/\rho)/\epsilon$ samples. Since \cref{lem: smoothness of measures} tells us that the density of $\mu_t$ satisfies $d(\mu_t) > \epsilon/32$, we can use \cref{lem:rejection_composition,lem:rejection} to conclude that \textsc{rBoost}$^*$ combined with \textsc{RejectionSampler} is $2\rho/(6T) + 2\rho/(6T) = 2\rho/(3T)$-replicable.
        Finally, by \cref{lem:rThreshold}, \textsc{rThreshold} is $\rho/(3T)$-replicable. Hence, by a union bound over the conditions, each iteration is $\rho/T$-replicable and union bounding over all $T$ iterations, the entire algorithm is $\rho$-replicable. 
    \end{proof}
\end{lemma}

\begin{lemma}[Sample complexity]
    \textsc{rMetaBoost} uses $m = \widetilde{O}\lp\frac{m_{\cW(\widetilde{\Theta}(\rho\gamma^2))}}{\epsilon\gamma^2} + \frac{1}{\rho^2\epsilon\gamma^3}\rp$ samples.
\end{lemma}
\begin{proof}
    For the sample complexity of a single iteration, we simply add up the sample complexities of all the subroutines:
\begin{itemize}[topsep=0pt]
    \item \textsc{rBoost}$^*$:
    Since we give it parameters $\rho_0=\rho/(6T)$,$\epsilon_0=1/16$, we get from \cref{thm:rBoost} that the sample complexity of \textsc{rBoost}$^*$ in a single iteration is 
    \[
        m_{\textsc{rBoost}^*}(\rho_0, \epsilon_0)= O\lp \frac{\log(\frac{T}{\rho\gamma^2})m_{\cW(\Theta(\frac{\rho\gamma^2}{T}))}}{\gamma^2} + \frac{\log(\frac{T}{\rho\gamma})T^2}{\rho^2\gamma^3}\rp
    \]
    Remark that the choice of constant $\epsilon_0$ removes all the dependence on $\epsilon$.

    \item \textsc{RejectionSampler}:
    To invoke \cref{lem:rejection} with failure probability $\rho/(6T)$ the number of samples used in each iteration is
    \[O\lp\frac{\log(\frac{T}{\rho})m_{\textsc{rBoost}^*}(\rho_0, \epsilon_0)}{\epsilon}\rp.\]
    \item \textsc{rThreshold}:
    To invoke \cref{lem:rThreshold} with replicability parameter $\rho/(3T)$ and failure probability $\rho/(24T)$ the number of samples used in each iteration is
    \[O(\frac{\log(\frac{T}{\rho})T^2}{\rho^2\epsilon}).\]
\end{itemize}
Remembering that the number of iterations is $T = O(\log(1/\epsilon))$ we get the total sample complexity to be
\begin{align*}
    &O\Bigg( T\Bigg( \frac{\log(\frac{T}{\rho})\log(\frac{T}{\rho\gamma^2})m_{\cW(\Theta(\frac{\rho\gamma^2}{T}))}}{\epsilon\gamma^2}+ \frac{\log(\frac{T}{\rho})\log(\frac{T}{\rho\gamma})T^2}{\rho^2\epsilon\gamma^3}
    + \frac{\log(\frac{T}{\rho})T^2}{\rho^2\epsilon}\Bigg)\Bigg)\\
    &= \widetilde{O}\lp \frac{m_{\cW(\widetilde{\Theta}(\rho\gamma^2))}}{\epsilon\gamma^2} + \frac{1}{\rho^2\epsilon\gamma^3} \rp \qedhere
\end{align*}
\end{proof}

\begin{lemma}[Failure probability]
    \textsc{rMetaBoost} fails with probability at most $9\rho/24 \le \rho$.
\end{lemma}
\begin{proof}
    The only sources of failure are the three subroutines. \textsc{RejectionSampler} fails with probability $\rho/(6T)$ in each iteration. \textsc{rThreshold} fails with probability $\rho/(24T)$ in each iteration. \textsc{rBoost}$^*$ fails with probability at most $\rho/(6T)$ in each iteration. Hence, the total failure probability of the algorithm is at most $9\rho/24 $.
\end{proof}

\section{Subroutines}\label{sec:subroutines}
In this section, we will present the replicable subroutines that the boosting algorithm uses. This includes \textsc{rThreshold} and \textsc{rBoost}$^*$. As mentioned earlier, we will not present \textsc{RejectionSampler} as we have made no changes to it, so we refer to \citet{impagliazzo2022reproducibility} for the description of this subroutine. We now move on to describe the two other subroutines.

\subsection{\textsc{rThreshold}}
In this section, we will describe \textsc{rThreshold} in more detail.
The pseudocode can be found in \cref{alg:rThreshold}. The purpose of this algorithm is to make a replicable test to see if $\E[\phi(X)] > z$ for some threshold $z\in (0,1)$ and $\phi:\cX \to [0,1]$. In the original version of \textsc{rBoost}, this is done by replicably simulating a statistical query for estimating $\E[\phi(X)]$, with an additive error of order $z$. However, for a threshold check it suffices to have a multiplicative error when estimating $\E[\phi(X)]$, which means we can do a better analysis by using a Chernoff bound.

\begin{algorithm}
    \caption{\textsc{rThreshold}$(S, z, \phi)$}\label{alg:rThreshold}
    \begin{algorithmic}[1]
        \REQUIRE Samples $S=(x_1, \ldots, x_m)$ drawn from $\cD$, threshold $z$, function $\phi: \cX \rightarrow [0,1]$.
        \ENSURE Bit $b$ being a guess, whether $\E_{x\sim \cD}[\phi(x)] > z$.
        \STATE $z_0 \gets_{r} [\frac{3}{4}z, \frac{3}{2}z]$\hfill$\triangleright$ Chosen uniformly at random
        \STATE $\overline{\phi(S)} \gets \frac{1}{m}\sum_{i=1}^m\phi(x_i)$
        \STATE \textbf{Return:} $b = \ind\left\{\overline{\phi(S)} > z_0\right\}$
    \end{algorithmic}
\end{algorithm}

We will now prove the guarantee of \textsc{rThreshold}. For convenience, we restate the guarantee here.

\begin{reptheorem}{lem:rThreshold}
\rThreshold
\end{reptheorem}
\begin{proof}
    It is sufficient to set $m \geq \frac{700\log(1/\delta)}{z\rho^2}$.
    We will first prove the first bullet point. Hence, assume that  $\E_{x\sim \cD}[\phi(x)] \leq z/2$. We then bound the following probability:
    \begin{align*}
        \p\left[b = 1\right]
        & = \p\left[\overline{\phi(S)} > z_0\right] \leq \p\left[\overline{\phi(S)} > \frac{3}{4}z\right]\\
        & = \p\left[\sum_{i=1}^m\phi(x_i) > (1 + \frac{1}{2})m(z/2)\right].
    \end{align*}
    Since the assumption states that $\E_{x\sim \cD}[\phi(x)] \leq z/2$, we can use a Chernoff bound to bound the above probability by
    \begin{align*}
        \exp\lp -mz/24\rp
        \leq \exp(-\log(1/\delta)/\rho^2)
        = \delta^{1/\rho^2} \leq \delta.
    \end{align*}
    We move on to the second bullet point. Hence, we now assume $\E_{x\sim \cD}[\phi(x)] \geq 2z$ and bound the following probability:
    \begin{align*}
        \p\left[b = 0\right]
        & = \p\left[\overline{\phi(S)} \leq z_0\right] \leq \p\left[\overline{\phi(S)} \leq \frac{3}{2}z\right]\\ 
        & = \p\left[\sum_{i=1}^m\phi(x_i) \leq (1 - \frac{1}{4})m(2z)\right].
    \end{align*}
    Again, since $\E_{x\sim \cD}[\phi(x)] \geq 2z$, we can use a Chernoff bound to bound the above probability by
    \begin{align*}
        \exp\lp -mz/16\rp
        \leq \exp(-\log(1/\delta)/\rho^2)
        = \delta^{1/\rho^2} \leq \delta.
    \end{align*}
    We will now show that \textsc{rThreshold} is $\rho$-replicable by considering two different runs of the algorithm with common randomness. Let $S_1, S_2 \sim \cD^{m}$ be the two sequences of samples used in the two runs. Assuming that neither of the runs fail, they will always output the same bit if $\E_{x\sim \cD}[\phi(x)] \notin [z/2, 2z]$, so we can safely assume that this is not the case. Now, we will bound the probability that $\overline{\phi(S_i)}$ deviates too much from $\E_{x\sim \cD}[\phi(x)]$. So we bound the following probabilities using Chernoff bounds and the assumption that $\E_{x\sim \cD}[\phi(x)] \in [z/2, 2z]$:
    \begin{align*}
        &\p\left[\overline{\phi(S_i)} - \E_{x\sim \cD}[\phi(x)]\geq 3z\rho/16\right] \\
        & \leq \exp(-3z\rho^2m/2048) \leq \delta \leq \rho/8
    \shortintertext{and}
        &\p\left[\overline{\phi(S_i)} - \E_{x\sim \cD}[\phi(x)] \leq -3z\rho/16 \right] \\
        & \leq \exp(-9z\rho^2m/4096) \leq \delta \leq \rho/8.
    \end{align*}
    Using a union bound, we can conclude that
    \[\p\left[ \left\lvert \overline{\phi(S_i)} - \E_{x\sim \cD}[\phi(x)]\right\rvert \geq 3z\rho / 16 \right] \leq \rho/4.\]
    We can therefore conclude that with high probability, the two estimates will be close to each other. That is,
    \[\p\left[ \left\lvert \overline{\phi(S_1)} - \overline{\phi(S_2)} \right\rvert \geq 3z\rho/8 \right] \leq \rho/2. \]
    So assume for now, that the two estimates are within $3z\rho/8$ of each other. Then the two runs will only give different outputs if the random split $z_0$ is chosen between them. The probability of this happening is at most the distance between the estimates divided by the total range of $z_0$, which is at most
    \[ \frac{3z\rho/8}{3z/2 - 3z/4} = \rho/2.\]
    So the probability that the two runs output different bits is at most $\rho/2 + \rho/2 = \rho$. Therefore, the algorithm is $\rho$-replicable.
\end{proof}

\subsection{\textsc{rBoost}$^*$}

We now move on to discuss in more detail why the modifications in \textsc{rBoost}$^*$ preserve correctness. The modified version can be seen in \cref{alg:rBoost}. The only two modifications can be found in line \ref{alg:rBoost:divides} and \ref{alg:rBoost:threshold}.
In line \ref{alg:rBoost:threshold} we have substituted a statistical query with our threshold check, and in line \ref{alg:rBoost:divides} we have inserted an if-statement to only do the threshold check every $1/\gamma$ iteration. In this algorithm, the threshold check gives the same guarantees as the statistical query, and it will therefore not affect correctness. However, the introduction of the if-statement could lead to two kinds of errors, since we do not check the value of $d(\mu_t)$ in every iteration. First, it could be that \textsc{RejectionSampler} fails, since it needs $d(\mu_t)$ to be large. Second, it could be that the number of iterations is increased, since the algorithm does not detect immediately when the density becomes small. To show that these events are not problematic, we first show that the densities do not decrease too much over $1/\gamma$ iterations.

\begin{lemma}\label{lem: scaled measures}
    Let $T_0$ denote the number of iterations that \textsc{rBoost}$^*$ runs for. Then for all $t \in [T_0]$ and $k\le \max\{\lfloor{1/\gamma}\rfloor,T_0-t\}$ that $d(\mu_{t+k})\ge d(\mu_t)/2$.
\end{lemma}
\begin{proof}
    Let $x\in \cX$. Then by the recursive definition of the $g_t$'s, we have $\mu_{t+1}(x) \ge (1-\gamma)^{1/2}\mu_t(x)$. Inductively, we get 
    \begin{align*}
        \mu_{t+k}(x) & \ge \mu_{t}(x)(1-\gamma)^{k / 2}\ge  \mu_{t}(1-\gamma)^{\lfloor \frac{1}{\gamma}\rfloor / 2} \\
        & \ge  \mu_{t}(x)\lp 1-\gamma \lfloor \frac{1}{\gamma}\rfloor/ 2 \rp \ge \frac{1}{2}\mu_{t}(x),
    \end{align*}
    where we use Bernoulli's inequality which applies since $-\gamma>-1$. Taking the expectation with respect to $\cD$ on both sides yields the desired conclusion. 
\end{proof}

\begin{reptheorem}{thm:rBoost}
    \rBoost
\end{reptheorem}

\begin{proof}
First, we argue that the \textsc{RejectionSampler} succeeds with high probability. Observe that due to \cref{lem: scaled measures}, the densities are at most halved in \textsc{rBoost}$^*$ compared to the original version of \textsc{rBoost}. Due to \cref{lem:rejection} we therefore only need to use twice as many samples in the rejection sampler for it to still succeed.

Next, we will argue that the number of iterations remains the same as in \textsc{rBoost} up to constant factors. The number of iterations is bounded in \citet{Smoothboost} by showing that for any $\kappa > 0$, there is some $t$ within the first $O(\frac{1}{\kappa\gamma^2})$ iterations such that $d(\mu_t) < \kappa$. This result also applies to \textsc{rBoost}$^*$. However, we will not repeat the proof here. 

We can then conclude that $d(\mu_t)$ will fall below $\epsilon/8$ within the first $O(\frac{1}{\epsilon\gamma^2})$ iterations. Since the threshold check in line \ref{alg:rBoost:threshold} always realizes when $d(\mu_t) \leq \epsilon/4$ (see \cref{lem:rThreshold}), and it takes $1/\gamma$ iterations to further decrease the density from $\epsilon/4$ to $\epsilon/8$, \textsc{rThreshold} will always have terminated the loop before reaching density $\epsilon/8$. Hence, the number of iterations in \textsc{rBoost}$^*$ is still $T_0 = O(\frac{1}{\epsilon\gamma^2})$.

We now argue, that replicability is preserved. Since this argument is almost identical to the proof of \cref{lem:replicability}, we will only describe what differs in this analysis. First, the weak learner is  called $T_0$ times, and therefore it needs replicability parameter $\rho/(6T_0) = \Theta(\rho\epsilon\gamma^2)$. Second, \textsc{rThreshold} is called $\gamma T_0$ times and therefore needs replicability parameter $\rho/(6\gamma T_0) = \rho\epsilon\gamma/6$, which it achieves due to \cref{lem:rThreshold}. Thus, our modifications preserve replicability. 

Finally, we calculate the sample complexity.
The \textsc{RejectionSampler} uses $O(\log(\frac{T_0}{\rho})m_{\cW(\Theta(\rho\epsilon\gamma^2))}/\epsilon)$ samples for each call, and it is called $T_0$ times.
Meanwhile, \textsc{rThreshold} uses $O(\frac{\log(\gamma T_0/\rho)}{\rho^2\epsilon^3\gamma^2})$ samples for each call, and is called $\gamma T_0$ times. Hence, the total sample complexity is
\begin{align*}
&O\lp T_0\frac{\log(\frac{T_0}{\rho})m_{\cW(\Theta(\rho\epsilon\gamma^2))}}{\epsilon} + \gamma T_0\frac{\log(\frac{\gamma T_0}{\rho})}{\rho^2\epsilon^3\gamma^2}\rp\\
&=O\lp \frac{\log(\frac{1}{\rho\epsilon\gamma^2})m_{\cW(\Theta(\rho\epsilon\gamma^2))}}{\epsilon^2\gamma^2} + \frac{\log(\frac{1}{\rho\epsilon\gamma})}{\rho^2\epsilon^4\gamma^3}\rp\\
&=\widetilde{O}\lp\frac{m_{\cW(\Theta(\rho\epsilon\gamma^2))}}{\epsilon^2\gamma^2} + \frac{1}{\rho^2\epsilon^4\gamma^3}\rp. \qedhere
\end{align*}
\end{proof}

\section*{Acknowledgements}{All authors are supported by the European Union (ERC, TUCLA, 101125203). Views and opinions expressed are however those of the author(s) only and do not necessarily reflect those of the European Union or the European Research Council. Neither the European Union nor the granting.}

\bibliography{refs.bib}
\bibliographystyle{abbrvnat}

\end{document}